\documentclass{article}
\usepackage{microtype}
\usepackage{graphicx}
\usepackage{subfigure}
\usepackage{amsmath}
\usepackage{amsthm}
\usepackage{amssymb}
\usepackage{bm}
\usepackage{algorithm}

\usepackage{xcolor}
\usepackage{tikz}
\usepackage{mathtools}
\usepackage{extarrows}
\newtheorem{theorem}{Theorem}[section]

\newtheorem{prop}[theorem]{Proposition}

\newtheorem{assumption}[theorem]{Assumption}
\newtheorem{thm}[theorem]{Theorem}
\newtheorem{lem}[theorem]{Lemma}
\newcommand*{\QEDA}{\hfill\ensuremath{\square}}
\usepackage{booktabs} 

\usepackage{hyperref}

\usepackage[accepted]{icml2023}

\icmltitlerunning{Reweighted Interacting Langevin Diffusions}

\begin{document}

\twocolumn[
\icmltitle{Reweighted Interacting Langevin Diffusions: an Accelerated Sampling Method for Optimization}



\icmlsetsymbol{equal}{*}

\begin{icmlauthorlist}
\icmlauthor{Junlong Lyu}{hwnh}
\icmlauthor{Zhitang Chen}{hwnh}
\icmlauthor{Wenlong Lyu}{hwnc}
\icmlauthor{Jianye Hao}{hwnc}

\end{icmlauthorlist}
\icmlaffiliation{hwnh}{Huawei Noah's Ark Lab, Hong Kong SAR, China.}
\icmlaffiliation{hwnc}{Huawei Noah's Ark Lab, China.}

\icmlcorrespondingauthor{Junlong Lyu}{lyujunlong@huawei.com}

\icmlkeywords{Optimization, Stochastic Process, Functional Analysis, Genetic Algorithm, ICML}

\vskip 0.3in
]



\printAffiliationsAndNotice{\icmlEqualContribution} 

\begin{abstract}
    We proposed a new technique to accelerate sampling methods for solving difficult optimization problems. Our method investigates the intrinsic connection between posterior distribution sampling and optimization with Langevin dynamics, and then we propose an interacting particle scheme that approximates a Reweighted Interacting Langevin Diffusion system (RILD). The underlying system is designed by adding a multiplicative source term into the classical Langevin operator, leading to a higher convergence rate and a more concentrated invariant measure. We analyze the convergence rate of our algorithm  and  the improvement compared to existing results in the asymptotic situation. We also design various tests to verify our theoretical results, showing the advantages of accelerating convergence and breaking through barriers of suspicious local minimums, especially in high-dimensional non-convex settings. Our algorithms and analysis shed some light on combining gradient and genetic algorithms using Partial Differential Equations (PDEs) with provable guarantees.
\end{abstract}

\section{Introduction}
Optimization methods have natural connections to sampling methods  \cite{Dalalyan2017FurtherAS} . Consider a potential function $V: \mathbb{R}^d \to \mathbb{R}$ that is twice differentiable and has only one simple global minimum ${\bm x}^\ast = \arg \min_{x \in \mathbb{R}^d} V({\bm x})$ for simplicity.   
The first order Langevin Dynamics is defined by the following Stochastic Differential Equation (SDE):
\begin{equation}\label{langevin-dynamics}
    d{\bm x}_t = -\nabla V({\bm x}_t) + \sigma d\mathcal{B}_t,
\end{equation}
where $\sigma>0$ is the diffusion parameter proportional to the square of the temperature parameter, and $\{\mathcal{B}_t \}_{t\ge 0}$ is the standard Brownian motion in $\mathbb{R}^d$. Under certain assumptions on the drift coefficient $\nabla V$, it was shown that the distribution of ${\bm x}_t$ in Eq.\ref{langevin-dynamics} converges exponentially to its stationary distribution, the Gibbs measure $\nu_\sigma({\bm x}) \propto \exp( -2 V({\bm x}) /\sigma^2)$. With $\sigma \to 0$, $\nu_\sigma$ concentrates on the global minimum $x^\ast$, leading to various Langevin dynamics based algorithms like Gradient Langevin Dynamics (GLD) \cite{Xu2017GlobalCO,Dalalyan2014TheoreticalGF}, Stochastic Gradient Langevin Dynamics (SGLD) \cite{Welling2011BayesianLV} etc.

While Langevin dynamics based algorithms handle optimization tasks with slightly nonconvex potential function $V$ effectively  \cite{Raginsky2017NonconvexLV, Zhang2017AHT}, it encounters difficulties with global optimization for highly nonconvex $V$, as it requires $\sigma_n \to 0$ to guarantee convergence. However, a small $\sigma_n$ makes the probability of jumping across a certain potential barrier drastically diminish. Conventional Markov Chain Monte Carlo (MCMC) methods were proposed to alleviate such problems, such as 
parallel tempering \cite{Swendsen1986ReplicaMC,Geyer1991MarkovCM}, flat histogram algorithms \cite{Berg1991MulticanonicalAF} and simulated annealing \cite{Kirkpatrick1983OptimizationBS}. Among them, Simulated Annealing (SA) provides a reasonable way to handle this problem by maintaining a population of particles to sequentially approximate $\nu_{\tilde \sigma_n}({\bm x})$  with $\tilde \sigma_n \to 0$. However, it requires numerious evaluations of $\nabla V$, which can be computationally costly, making it less attractive for optimization tasks. Therefore in this paper, we aim at accelerating Langevin dynamics based algorithms by modifying the underlying continuous Langevin dynamics.

Recently, an Ensemble Kalman Sampler (EKS) \cite{GarbunoIigo2019InteractingLD,Schillings2016AnalysisOT} method was suggested to accelerate the convergence rate toward the equilibrium $\nu_\sigma({\bm x})$. This method approximates a preconditioning-modified Langevin dynamic that shares the same equilibrium with the original one, but enjoys a faster convergence rate for ill-posed problems. Besides, this method has another benefit when the potential function $V({\bm x}) = \Vert
\mathcal{G}({\bm x}) - y\Vert^2$ is a least square functional \cite{Engl1996RegularizationOI}, i.e. the gradient of $V$, or the Jacobean matrix of $\mathcal{G}$ does not need to be calculated explicitly, as those terms involving them can be approximated by the $0^{th}$ order information in current steps, see Eq. ~\eqref{ls-langevin-diffusion-der-free}. 

However, the EKS method has little improvement when solving highly non-convex optimizations. Besides,  as our purposes is to find the global minimum of $V$, restricting the stationary distribution of the underlying dynamic to be precisely $\nu_\sigma({\bm x})$ is not necessary. Suppose we have another  $\sigma$-dependent dynamic with limit distribution approximate $\delta_{{\bm x}^\ast}({\bm x})$
, the Delta distribution at the minimum point $\bm x\ast$, then this dynamic can also be chosen to approximate the global minimum. This inspires us to further modify the Langevin dynamics for faster convergence. Specifically, we modify the Fokker-Planck Equation related to Langevin Dynamics by adding a linear source term, which can be proven, by the spectral approach \cite{Pankov2001IntroductionTS}, to improve the convergence rate and the mass concentration near the global minimum of the invariant measure. The new process belongs to a type of nonlinear operators called Feynman-Kac Semigroup which was developed in Large Deviation Theory for calculating generating functions \cite{Varadhan2010LARGED} and also used in important practical applications such as the Diffusion Monte Carlo (DMC) method \cite{Foulkes2001QuantumMC}. 
To design a practical algorithm, we use the Interacting Particle methods \cite{Moral2000BranchingAI,Moral2013MeanFS}, introducing the reweighting and resampling technique to simulate the effect of this source term and get the so-called Reweighted Interacting Langevin Diffusion (RILD) algorithm. This is, as far as we know, the first time using Interacting Particle methods to solve optimization tasks. Many numerical experiments are tested to show the effect in accelerating convergence and flatten the potential barriers.

The main contributions are summarized as follows:
\begin{itemize}
    \item We provide a simple but effective way to modify the Langevin Dynamics in Section \ref{sec:continuous}  for faster convergence and flatter potential barriers. 
    \item We provide a feasible discretization way to design algorithms in Section \ref{sec:discrete}, and compare the new algorithm Alg. \ref{alg:RILD}  with several existing methods, showing its advantages.
    \item We provide theoretical results in Section \ref{theoreticalproperty} based on spectral analysis to guarantee the advantages.
\end{itemize}

\section{Preliminary}
\subsection{Overdamped Langevin Dynamics}
In this section, we introduce the classical Langevin Dynamics which is related to GLD and SGLD algorithms.

Suppose $V: \mathbb{R}^d \to \mathbb{R}$ that is twice differentiable and has a simple global minimum ${\bm x}^\ast = \arg \min_{x \in \mathbb{R}^d} V({\bm x})$ for simplicity.

The overdamped Langevin Dynamics is defined as in \ref{langevin-dynamics}. 
Such a dynamic has plenty of relationships with the Fokker-Planck equation. Denote \begin{equation}\mathcal{L}\cdot := -\langle \nabla V , \nabla \cdot \rangle + \frac{\sigma^2}{2} \Delta \cdot\end{equation} the infinitesimal generator \cite{ksendal1987StochasticDE} of the Markov process $({\bm x}_t)_{t\ge 0}$, and its $L^2$ adjoint operator:
\begin{equation} \mathcal{L}^\dagger \cdot := \text{div}( \nabla V \cdot + \frac{\sigma^2}{2} \nabla \cdot ) \end{equation}
Let us assume the law of  ${\bm x}_t$ at time $t$ has a density $p_t( x)$ for Lebesgure measure. Then $p_t({\bm x})$ satisfies the Fokker-Planck equation:
\begin{equation}\frac{\partial}{\partial t}p_t({\bm x}) =\mathcal{L}^\dagger p_t({\bm x}) = \text{div} \big( p_t({\bm x}) \nabla V({\bm x})  + \frac{\sigma^2}{2} \nabla p_t({\bm x}) \big) \label{Langevin-Diffusion}\end{equation}
We may denote $p_t({\bm x}) = e^{\mathcal{L}^\dagger t} p_0({\bm x})$ where we denote $p_0({\bm x}) $ the initial distribution of overdamped Langevin dynamics, and it is well-known that the Markov operator $e^{\mathcal{L}^\dagger t}$ admits a unique invariant probability measure $\nu({\bm x}) = Z_\nu^{-1}e^{-2\sigma^{-2}V({\bm x})}, Z_\nu = \int_{\mathbb{R}^d}e^{-2\sigma^{-2}V({\bm x})} dx$. The rate converging to $\nu({\bm x})$ has been wildly studied,

\begin{prop}[Proposition 2.3 in \citet{Lelivre2016PartialDE}]
    Under specific condition (Assumption \ref{Asmp:VW}),  \footnote{ We define $L^2(\mu):= \{f: \Vert f \Vert_{L^2(\mu)} < \infty\}$,  where $\langle f,g \rangle_{L^2(\mu)} := \int_{\mathbb{R}^d} f(x)g(x) \mu(x) dx$,  $\Vert f \Vert_{L^2(\mu)} = \langle f,f \rangle_{L^2(\mu)}^{\frac{1}{2}}$}for all $p_0$ such that $p_0 / \nu \in L^2(\nu)$, and for all $t \ge 0$,
    $$\Vert e^{\mathcal{L}^\dagger t}p_0/ \nu - 1\Vert_{L^2(\nu)} \le \Vert p_0 / \nu - 1\Vert_{L^2(\nu)} e^{-\delta t},$$
    where $\delta$ is the first non-zero eigenvalue $-\lambda_1$ of the operator $-\mathcal{L}$ on $L^2(\nu)$.
\end{prop}
Note that $\lambda_1$ is real because $\mathcal{L}$ is self-adjoint\footnote{We say a linear operator $A$ is self-adjoint on a Hilbert space $\{ \mathcal{H}, \langle \cdot \rangle_\mathcal{H} \}$, if for any $f, g \in \mathcal{H}, \langle f, Ag \rangle_\mathcal{H} = \langle Af, g \rangle_\mathcal{H}$} on $L^2(\nu)$.

When for sampling purposes, individual samples are sampled from independent paths by Monte Carlo Markov Chain (MCMC, \cite{Berg2004MarkovCM}) corresponding to the discretized Langevin dynamics. This method is somehow slow for large dimensional, non-convex problems. Such a weakness make it less attractive for optimization purposes. Fortunately, we can explore the connection between individual samples, as they did in \citet{GarbunoIigo2019InteractingLD}. We will introduce their work in the next subsection.

\subsection{Interacting Langevin Diffusion}
In \citet{GarbunoIigo2019InteractingLD}, the authors introduced a modified Langevin dynamics and analyzed its property, proving its superiority in designing sampling algorithms. We now introduce the main result of their work.

The convergence rate of classical Langevin dynamic based algorithm can be really slow when $V$ varies highly anisotropic, . A common approach for canceling out this effect is to introduce a $d\times d$ preconditioning positive semi-definite matrix $C$ in the corresponding gradient scheme,
\begin{equation} d {\bm x}_t = -C\nabla V({\bm x}_t)dt + \sigma\sqrt{C}d\mathcal{B}_t.\end{equation}
Here $\sqrt{C}$ can be any $d\times r$ real matrix $U$ such that $U U^T = C$ (Note that in this case, $W_t$ can be reduced to $r$-dimensional standard Brownian motion as the essential rank of $C$ is no larger than r).

The corresponding Fokker-Planck equation now becomes
\begin{equation}\frac{\partial}{\partial t}p_t({\bm x}) = \text{div} \big(p_t({\bm x})  C\nabla V({\bm x}) +  \frac{\sigma^2}{2} C \nabla p_t({\bm x}) \big) \label{Preconditioned-Langevin}\end{equation}
and the infinitesimal generator
\begin{align}\mathcal{L}_C\cdot &:= -\langle C\nabla V, \nabla \cdot \rangle + \frac{\sigma^2}{2} \text{div}(C\nabla \cdot),\\
\mathcal{L}_C^\dagger\cdot &:= \text{div} \big( C\nabla V \cdot  +  \frac{\sigma^2}{2} C \nabla\cdot \big).\end{align}
One can easily verify that $e^{-2\sigma^{-2}V({\bm x})}$ is the invariant measure to the above system for all semi-definite $C$, and the only one positive invariant measure if $C$ is strictly positive definite.

To find a suitable $C$ is of general interest. One of the best choices is to let $C = \text{Hess} V$, as a counterpart of Newton's scheme in optimization, which is unfriendly for computation. One intrinsic benefit of choosing $\text{Hess}V$ is its affine-invariant property when designing numerical schemes.

    

Such a property can also be preserved if we take $C = \mathcal{C}(p)$, the covariance matrix under the probability measure $p({\bm x})$, \begin{equation} \mathcal{C}(p):= \int_{\mathbb{R}^d} \big({\bm x}- m(p)\big) \otimes \big({\bm x}- m(p)\big) p({\bm x}) d{\bm x},\end{equation} \begin{equation}m(p):= \int_{\mathbb{R}^d} {\bm x}p({\bm x}) dx.\end{equation} 

The dynamic now becomes a nonlinear flow
\begin{align} d{\bm x}_t &= -\mathcal{C}(p_t) \nabla V({\bm x}_t)dt + \sigma \sqrt{\mathcal{C}(p_t)} d\mathcal{B}_t,\\
\label{Inteacting-Langevin-Diffusion}\frac{\partial}{\partial t}p_t &= \text{div} \big(p_t  \mathcal{C}(p_t)\nabla V +  \frac{\sigma^2}{2} \mathcal{C}(p_t) \nabla p_t \big)\end{align}
To simulate from such a mean-field model, the finite interacting particle system $\mathcal{X}_t = \{ {\bm x}^i_t \}_{i=1}^N$ is introduced:
\begin{equation} d{\bm x}^i_t = - \mathcal{C}(\mathcal{X}_t)\nabla V({\bm x}_t^i)dt + \sigma \sqrt{\mathcal{C}(\mathcal{X}_t)}d \mathcal{B}_t^i, \label{particle-langevin-diffusion}\end{equation}
where $\mathcal{B}_t^i$ are i.i.d. standard Brownian motions, and 
\begin{align*} \label{covarianceupdate} &\mathcal{C}(\mathcal{X}_t) := {\bm X}_t {\bm X}_t^T, \sqrt{\mathcal{C}(\mathcal{X}_t)} := {\bm X}_t,\\ 
&{\bm X}_t := \frac{1}{\sqrt{N}}\large( {\bm x}_t^1 - \bar {\bm x}_t, \cdots, {\bm x}_t^N - \bar {\bm x}_t\large), \bar {\bm x}_t = \frac{1}{N} \sum_{i=1}^N {\bm x}_t^i\end{align*}
Particles in this system are then no longer independent to each other, in contrast to independent simulations in SGLD algorithm.

Now suppose $V$ is a least squares functional\footnote{For any positive-definite matrix $A$, we define $\langle a, a'\rangle_A = \langle a, A^{-1} a'\rangle =  a^T A^{-1} a'$, and $\Vert a \Vert_A = \Vert A^{-\frac{1}{2}}a\Vert$.  } with Tikhonov-Phillips regularization \cite{Engl1996RegularizationOI}:
\begin{equation} V(x) = \frac{1}{2} \Vert y - \mathcal{G}(x)\Vert_\Gamma^2 + \frac{1}{2}\Vert x \Vert_{\Gamma_0}^2,\label{Least-square-functional}\end{equation}
where $y \in \mathbb{R}^k, \mathcal{G}: \mathbb{R}^d \to \mathbb{R}^k$ is the forward map, $\Gamma$ and $\Gamma_0$ are two positive definite matrixs. In this situation, the system (\ref{particle-langevin-diffusion}) can be re-written as
\begin{align} 
d{\bm x}^i_t =& - \frac{1}{N}\sum_{j=1}^N \left \langle D\mathcal{G}(\bm{x}^i_t)(\bm{x}^j_t - \bm{\bar x}_t), \mathcal{G}(\bm{x}^i_t) - y \right\rangle_\Gamma \bm{x}_t^j dt \nonumber\\
 &- \mathcal{C}(\mathcal{X}_t) \Gamma_0^{-1} \bm{x}_t^i dt + \sigma \sqrt{\mathcal{C}(\mathcal{X}_t)}d W_t^i. \label{ls-langevin-diffusion}\end{align}
Using the $1^{\text{st}}$ order Taylor approximation 
$$D\mathcal{G}(\bm{x}^i_t)(\bm{x}^j_t - \bm{\bar x}_t) \approx \mathcal{G}(\bm{x}^j_t) - \mathcal{\bar G}_t, \quad \mathcal{\bar G}_t := \frac{1}{N}\sum_{k=1}^N \mathcal{G}(\bm{x}^k_t),$$
one may approximate Eq. \ref{ls-langevin-diffusion} in a derivative-free manner as
\begin{align} 
d{\bm x}^i_t =& - \frac{1}{N}\sum_{j=1}^N \left \langle \mathcal{G}(\bm{x}^j_t) - \mathcal{\bar G}_t, \mathcal{G}(\bm{x}^i_t) - y \right\rangle_\Gamma \bm{x}_t^j dt \nonumber\\
 &- \mathcal{C}(\mathcal{X}_t) \Gamma_0^{-1} \bm{x}_t^i dt + \sigma \sqrt{\mathcal{C}(\mathcal{X}_t)}d W_t^i. \label{ls-langevin-diffusion-der-free}\end{align}

\section{Reweighted Interacting Langevin Diffusion}
\subsection{Continuous process analysis}\label{sec:continuous}
Keeping the invariant measure to be exactly $e^{-2\sigma^{-2}V({\bm x})}$ is not necessary for optimization. It would be preferable if a new process can converge faster to its invariant measure with its invariant measure still concentrating on the global optimum as $\sigma \to 0$.

Now we introduce an additional source term to modify Eq. \eqref{Preconditioned-Langevin}  into
\begin{align}\frac{\partial}{\partial t}\tilde p_t &= \text{div} \big(\tilde p_t C\nabla V +  \frac{\sigma^2}{2} C \nabla \tilde p_t \big) + W \tilde p_t \label{Inteacting-Langevin-Diffusion-reweight1}\\
p_t({\bm x}) &= \frac{\tilde p_t({\bm x})}{\int_{\tilde {\bm x}\in \mathbb{R}^d} \tilde p_t(\tilde {\bm x}) d \tilde {\bm x}}\label{Inteacting-Langevin-Diffusion-reweight2}
\end{align}

Here we mainly consider the situation when $C = I$ corresponding to Eq. \ref{Langevin-Diffusion}, or $C = \mathcal{C}(p_t)$ corresponding to Eq. \ref{Inteacting-Langevin-Diffusion}, but the analysis remains valid for all positive-definite $C$. We assume in this paper that the function $W$ is smooth and upper bounded. 

Let us look inside what the role $W$ plays in the evolution of this process. If we take $W$ as a function to $f$ such that $W({\bm x})$ becomes larger when $f({\bm x})$ becomes smaller, intuitively the ratio of the mass in the better-fitness region (closer to the global minimum) becomes larger, thus we expect the invariant measure concentrates more on the global optimum. 

Note that such a process (\ref{Inteacting-Langevin-Diffusion-reweight1}), (\ref{Inteacting-Langevin-Diffusion-reweight2})  is no longer a gradient flow structure, which does not preserve total mass,  thus a normalization is added to keep the total mass equal to 1. Such a normalizing process has been well-studied as the so-called Feynman-Kac Semigroup when considering the linear case: $C$ is a constant matrix. As the spectral analysis is wildly studied for linear operators \cite{Kato1966PerturbationTF}, and for numerical discretization the $C$ is fixed at each time step, we thus conduct analysis for any fixed matrix $C$. 

We introduce the solution operator corresponding to Eq. \ref{Inteacting-Langevin-Diffusion-reweight1} and Eq. \ref{Inteacting-Langevin-Diffusion-reweight2}: recall the infinite generator $\mathcal{L}^\dagger_C$, the solution in Eq. \ref{Inteacting-Langevin-Diffusion-reweight1} can be represented by $\tilde p_t = e^{t(\mathcal{L}^\dagger_C + W)}p_0$. We denote the corresponding reweighted operator, which is also well-known as Feynman-Kac Semigroup, as $\Phi^t_{\mathcal{L}_C+W}:$ 

$$ \Phi^t_{\mathcal{L}_C+W}(p_0) := \frac{ e^{t(\mathcal{L}^\dagger_C + W)} p_0}{\int_{\mathbb{R}^d}e^{t(\mathcal{L}^\dagger_C+W)} p_0(\tilde {\bm x}) d\tilde {\bm x} },$$

then $p_t =  \Phi^t_{\mathcal{L}_C+W}(p_0)$ is exactly the solution to Eq. \ref{Inteacting-Langevin-Diffusion-reweight2}.

The operator $\Phi^t_{\mathcal{L}_C+W}$ is nonlinear in general, but it has many similarity with the linear operator $e^{t(\mathcal{L}^\dagger_C + W)}$. In specific, the unique positive fixed-point $p_\infty$ of $\Phi^t_{\mathcal{L}_C+W}$ is proportional to the principle eigenfunction of $(\mathcal{L}^\dagger_C + W)$, and the convergence rate of $\Phi^t_{\mathcal{L}_C+W}(p_0)$ to $p_\infty$ is controlled by the spectral gap of $(\mathcal{L}_C + W)$. For the Feynman-Kac semigroup, one can refer to \citet{Ferr2017ErrorEO} for time-invariant case, \citet{Lyu2021ACI} for time-periodic case, and \citet{Moral2004FeynmanKacFG} for
systematical details.  

We are interested in what benefits the source term $W$ can contribute. We show that, when taking $W = \varepsilon m(V)$ for some small $\varepsilon >0$ where $m:\mathbb{R}\to\mathbb{R}$ is a monotonic decreasing function of $V$, this brings mainly two benefits:
\begin{itemize}
    \item the spectral gap to the operator $\mathcal{L}_C + W$ is larger than $\mathcal{L}_C$ considered in same functional space;
    \item the invariant measure of $ \Phi^t_{\mathcal{L}_C+W}$ concentrates more on the global minimum compared to the invariant measure of $e^{t(\mathcal{L}^\dagger_C)}$.
\end{itemize}
These two benefits show that, in the same noise level, process \eqref{Inteacting-Langevin-Diffusion-reweight2} can converge faster and have more concentrated invariant measure than process \eqref{Langevin-Diffusion} or \eqref{Inteacting-Diffusion-reweight}. We will further explain the benefits, giving proof in Section  \ref{Interacing-Langevin-Diffusion-Theory}.

\subsection{Discrete algorithm design}\label{sec:discrete}
Now let us use interacting particle methods \cite{Moral2000BranchingAI,Moral2013MeanFS} and mean field theory \cite{Kac1960ProbabilityAR} to design algorithms. We introduce the so-called Reweighted Interacting Langevin Diffusion algorithm, by simply approximating $p_t$ with a population of particles, and discretize the evolution of time by operator-splitting technique and forward-Euler scheme. 

First, we convert Eq. \ref{Inteacting-Langevin-Diffusion-reweight1} and \ref{Inteacting-Langevin-Diffusion-reweight2} into a discrete-time version: let $\tau >0$ be a fixed timestep:
$$p_{(n+1)\tau} = \frac{ e^{\tau(\mathcal{L}^\dagger_C + W)} p_{n\tau}}{\int_{\mathbb{R}^d}e^{\tau(\mathcal{L}^\dagger_C+W)} p_{n\tau}(\tilde {\bm x}) d\tilde {\bm x} }$$

then we use the operator splitting technique\footnote{Note that higher order splitting technique can be applied, such as Strange splitting \cite{Strang1968OnTC} $e^{\frac{\tau}{2} W} \circ e^{\tau \mathcal{L}^\dagger_C} \circ e^{\frac{\tau}{2} W} $ with splitting error $O(\tau^3)$. However, the approximation error can still be induced in later steps, thus we only choose a simplest one here.}: approximate $e^{\tau(\mathcal{L}^\dagger_C + W)}$ by $e^{\tau W}\circ e^{\tau \mathcal{L}^\dagger_C}$ with splitting error $O(\tau^2)$ .

Let's now approximate $e^{\tau \mathcal{L}^\dagger_C} p_{n\tau}$. Suppose $p_{n\tau}$ is approximated by a weighted empirical measure that can be expressed as  $\hat p_{n\tau} ({\bm x}) =  \sum_{i=1}^{N}w^i_n\delta_{{\bm x}_n^i}({\bm x})$ generated from a sort of sample-weight pair $\{ {\bm x}_n^i, w_n^i \}_{i=1}^N$. We use the simplest Euler-Maruyama \cite{Faniran2015NumericalSO} approximation,
\begin{align}
&e^{\tau \mathcal{L}^\dagger_C}\hat p_{n\tau}({\bm x}) \approx \frac{1}{N} \sum_{i=1}^{N}w_n^i\delta_{ {\bm x}_{n+1}^i}({\bm x}), \\
&{\bm x}_{n+1}^i = {\bm x}_n^i - C \nabla V({\bm x}_n^i)\tau + \sqrt{\tau\sigma^2C}\xi^i_n \\
&e^{\tau W} e^{\tau \mathcal{L}^\dagger_C}\hat p_{n\tau}({\bm x}) \approx  \frac{1}{N} \sum_{i=1}^{N} w_n^i e^{\tau W( {\bm x}_{n+1}^i)} \delta_{ {\bm x}_{n+1}^i}({\bm x}), \label{discretescheme-reweighted-Langevin}\end{align}
Here $C$ can be $I$ or the covariance matrix of current step: if we take the matrix $C$ to be the covariance matrix, it is computed as (denote $\Lambda = \text{diag}(w_n^i), \quad\bar {\bm x}_n = \frac{1}{N} \sum_{i=1}^N w_n^i {\bm x}_n^i$)
\begin{align} C = \mathcal{C}(\hat p_{n\tau}) =  {\bm X}_n \Lambda {\bm X}_n^T, \sqrt{C} = {\bm X}_n\Lambda^{\frac{1}{2}},  \\
{\bm X}_n := \frac{1}{\sqrt{N}}\large( {\bm x}_n^1 - \bar {\bm x}_n, \cdots, {\bm x}_n^N - \bar {\bm x}_n\large). \label{Reweight-Covariance-Update}\end{align}

Then, after normalizing, we get the natural approximation
\begin{align}  p_{(n+1)\tau}({\bm x}) \approx    \sum_{i=1}^{N}w^i_{n+1}\delta_{{\bm x}_{n+1}^i}({\bm x}), \\
 w_{n+1}^i = \frac{ w_n^i e^{\tau W( {\bm x}_{n+1}^i)}}{\sum_{j=1}^N w_n^je^{\tau W( {\bm x}_{n+1}^j)}} \label{weight-update}\end{align}
Note that elements in $\{w^i_{n}\}_{i=1}^N$ may be polarized, we use a resampling technique for better approximation: If $\frac{\max_i w^i_n}{\min_i w^i_n}$ reaches to a threshold, we resample the replicas $\{ \bm{x}_n^i \}_{i=1}^N $ according to the multinomial distribution associated with $\{w_n^i\}_{i=1}^N $, which defines a new set of replicas $\{ \bm{\tilde x}_n^i \}_{i=1}^N$ and the empirical distribution $$\tilde p_{n\tau}(\bm{x}) = \frac{1}{N}\sum_{i=1}^N \delta_{\bm{\tilde x}_n^i}(\bm{x}).$$
Then we replace $\hat p_{n\tau}$ by $\tilde p_{n\tau}$ in Eq. \ref{discretescheme-reweighted-Langevin} and conduct further computation.

When $V$ is a least square functional with Tikhonov-Phillips regularization like in Eq. \ref{Least-square-functional}, we can follow the same analysis, approximating $\mathcal{C}(\hat p_{n\tau}) \nabla V({\bm x}_n^i)$ in a similar derivative-free manner as in Eq. \ref{ls-langevin-diffusion-der-free}: ( $\mathcal{\bar G}_n := \frac{1}{N}\sum_{k=1}^N w_n^i\mathcal{G}(\bm{x}^k_n) $ )
\begin{align}\label{rw-langevin-diffusion-der-free}
    &\mathcal{C}(\hat p_{n\tau}) \nabla V({\bm x}_n^i)\\
    \approx &\frac{1}{N}\sum_{j=1}^N w_n^j\left \langle \mathcal{G}(\bm{x}^j_n) - \mathcal{\bar G}_n, \mathcal{G}(\bm{x}^i_n) - y \right\rangle_\Gamma \bm{x}_n^j
     +\mathcal{C}(\hat p_{n\tau}) \Gamma_0 \bm{x}_n^i. \nonumber 
\end{align}
 We now conclude in algorithm \ref{alg:RILD}. Note that we fix the population $N$, stepsize $\tau$, and noise level $\sigma$ for simple, but these can be adjusted dynamically for faster convergence.
 
Note that our methods  can also be interpreted as a mutation-selection genetic particle algorithm with MCMC mutations: the update rule for $\bm{x}_n^i \to \bm{x}_{n+1}^i$ can be regarded as mutation, and the resampling step can be regarded as selection. Such a type of algorithm acts like a ridge connecting Genetic Algorithm and PDEs, which can conduct better convergence analysis.

\begin{algorithm*}
\caption{Reweighted Interacting Langevin Diffusion (RILD) }\label{alg:RILD}
\begin{algorithmic}
\STATE {\bfseries Input:} Functional to minimize $V({\bm x})\in \mathbb{R}, x \in X \subseteq \mathbb{R}^d$, gradient $\nabla V({\bm x}) \in \mathbb{R}^d$, Fitness  $W(\bm{x})\in \mathbb{R}$, population size $N$, Maximum iteration $Iter$, i.i.d. d-dimensional Standard Gaussian samples $\{ \xi^i_n \}_{1\le i\le N, 0 \le n < Iter}$, Stepsize $\tau$, noise $\sigma$.
\STATE {\bfseries Onput:} Samples nearby global minimum.
\STATE Sample ${\bm x}_0^i, i = 1,\cdots,N$ under initial distribution $p_0$; $w_0^i = \frac{1}{N},i = 1,\cdots,N$, $n = 0$.
\WHILE {$n < Iter$}
\STATE Calculate $C$ by \eqref{Reweight-Covariance-Update} or just $C=I$, depends on if one use covariance modification.
\STATE If use covariance modification, directly calculate $C \nabla V({\bm x}_n^i)$ or approximate it by \eqref{rw-langevin-diffusion-der-free}, depends on if $V$ is a least square functional \eqref{Least-square-functional}.
\STATE ${\bm x}_{n+1}^i = {\bm x}_n^i - C \nabla V({\bm x}_n^i)\tau + \sqrt{\tau\sigma^2C}\xi^i_n$, $\quad w_{n+1}^i = \frac{ w_n^i e^{\tau W( {\bm x}_{n+1}^i)}}{\sum_{j=1}^N w_n^je^{\tau W( {\bm x}_{n+1}^j)}}$,  $\quad 1\le i \le N$

\IF {$\frac{\max_i w^i_{n+1}}{\min_i w^i_{n+1}} > Threshold$}
    \STATE Sample $\{ \bm{\tilde x}_{n+1}^i \}_{i=1}^N $ from $\{ \bm{x}_{n+1}^i \}_{i=1}^N $ according to the multinomial probability $\{ w_{n+1}^i\}_{i=1}^N $.
    \STATE $\bm{x}_{n+1}^i = \bm{\tilde x}_{n+1}^i, w_{n+1}^i = \frac{1}{N}, 1\le i \le N$
\ENDIF
\STATE $n = n+1$
\ENDWHILE
\STATE {\bf Output} $\{ \bm{x}_{n}^i \}_{i=1}^N$
\end{algorithmic}
\end{algorithm*}

\subsection{Gradient-Free variants}

In practice, many problems are hard to get the exact gradients for optimizing: gradient information is infeasible, or computationally expensive. We thus suggest a gradient-free variant, which corresponds to a process only with diffusion and source term. We will prove that such a  process can exponentially converge to its invariant measure, and it also has an invariant measure that generally concentrates on the global minimum as $\sigma \to 0$.

We introduce the formula of the modified process as
\begin{align}\label{Inteacting-Diffusion-reweight}\frac{\partial}{\partial t}\tilde p_t &= \text{div} \big(  \frac{\sigma^2}{2} C \nabla \tilde p_t \big) + W \tilde p_t \\
p_t({\bm x}) &= \frac{\tilde p_t({\bm x})}{\int_{\tilde {\bm x}\in \mathbb{R}^d} \tilde p_t(\tilde {\bm x}) d \tilde {\bm x}}
\end{align}

Note that we only delete the term related to $\nabla V$,  the term $W$ that can still be chosen to relate to $V$. We will state in Thm. \ref{Thm:Gradfree} that, the system will finally converge to a distribution concentrating on the maximum of $W$.

The discrete algorithm is designed similarly as in Algorithm \ref{alg:RILD}, but just ignores the gradient term $C\nabla V$. This can be seen by taking $V(\bm x) \equiv \text{const}$.

\section{Theoretical properties} \label{theoreticalproperty}

In this section, we state the main theoretical results associated with our RILD algorithm \ref{alg:RILD}. The proofs are offered in Appendix \ref{sec:proofs}.

\subsection{Spectral Gap enhancement of the reweighting modification}\label{Interacing-Langevin-Diffusion-Theory}
Now we analyze how $W$ helps in improving the convergence rate, as well as the sharpness of the invariant measure. 

First let us restrict the comparison in the $L^2(\nu)$, where $\nu(x) = e^{-\frac{2V(x)}{\sigma^2}}$. The reason to choose this space is mainly because the following property,
\begin{lem}\label{Lem:symmetric}
    For any  positive definite matrix $C$, $L_C + W$ and $L_C$ are self-adjoint over $L^2(\nu)$.
\end{lem}
To continue our analysis, we need to make following assumption for $V$ and $W$, which is necessary for $L_C+W$ and $L_C$ have a discrete and up-bounded spectrum (see \citet{Pankov2001IntroductionTS}).

\begin{assumption}\label{Asmp:VW} The functions $V$ and $W$ are assumed to satisfy:
    $$\lim_{|\bm x| \to \infty} V = +\infty, W < A \text{ for a constant A} \in \mathbb{R},$$
    and
    $$\lim_{\vert \bm x\vert \to \infty} \frac{ \vert \nabla V\vert^2}{2\sigma^2} - \frac{\Delta V}{2} - W =   \lim_{\vert \bm x\vert \to \infty} \frac{ \vert \nabla V\vert^2}{2\sigma^2} - \frac{\Delta V}{2} = +\infty$$
\end{assumption}

Next, we need to prove that the Feynman-Kac semigroup $\Phi^t_{\mathcal{L}_C+W}$ does map any initial density $p_0 \in L^2(\nu)$ to a limit density. Such a result is concluded as follows:

\begin{thm}\label{Thm:convergence1}
    Under Asmp. \ref{Asmp:VW}, For any $C$ that is positive-definite, there exists a principle eigenvalue $\lambda_0$ of $\mathcal{L}_C + W$ over $L^2(\nu)$ with the corresponding normalized eigenfunction $\phi({\bm x})$. Furthermore, for any positive density $p_0$, the normalized probability $p_t := \Phi^t_{\mathcal{L}_C+W}(p_0)$ has a limit equal to $\phi({\bm x}) \nu({\bm x}) $, that is, $\lim_{t\to \infty} ||p_t/\nu -  \phi||_{L^2(\nu)} = 0$.
\end{thm}

\begin{thm}\label{Thm:convergence2}
     Under the same assumption as in Lem. \ref{Thm:convergence1}, the convergence rate of the system \eqref{Inteacting-Langevin-Diffusion-reweight1}, \eqref{Inteacting-Langevin-Diffusion-reweight2}  can be evaluated by the spectral gap of $\mathcal{L}_C + W$ over $L^2(\nu)$:  let $\lambda_0$ and $\lambda_1$ be the first two eigenvalues of $\mathcal{L}_C + W$, then  $ ||p_t/\nu -  \phi||_{L^2(\nu)} \le C || p_0/\nu - \phi||_{L^2(\nu)} e^{-(\lambda_0 - \lambda_1)t}$.
\end{thm}

Next, let us analyze how the convergence speed of $\Phi^t_{\mathcal{L}_C+W}p_0$ is improved compared to the original process $e^{t\mathcal{L}_C}(p_0)$, which is heavily related to the spectral gap. To exactly analyze the spectral gap to a differential operator is hard in general. Many existing analysis of spectral gap only consider the simplest case: $\nabla V$ is a constant matrix. In contrast to these existing techniques, we analyze the contribution of $W$ to the spectral gap by perturbation theory:  we will prove that the new process has a better convergence rate compared to the old one, if $W = \varepsilon m(V)$ for a small $\varepsilon >0$, where $m: \mathbb{R} \to \mathbb{R}$ is a monotonic decreasing function to $V$.

\begin{thm}\label{Thm:SpectralGapEnhancement}
    Suppose in addition $V$ satisfies the condition the same as in \citet{Nier2004QuantitativeAO}. If we take $W = \varepsilon m(V)$, consider under the space $L^2(\nu)$, where $\nu(x) = e^{-\frac{2 V(x)}{\sigma^2}}$, then when $\sigma$ small enough,  the spectral gap of $\mathcal{L}_C+\varepsilon m(V)$ is locally increasing  v.s. $\varepsilon$ for small enough $\varepsilon>0$. Besides, the principle eigenfunction of $\mathcal{L}_C+\varepsilon m(V)$ concentrates more on global minimum than the principle eigenfunction of $\mathcal{L}_C$ for small enough $\varepsilon>0$.
\end{thm}

\subsection{Convergence of the Gradient-Free variants}

In the gradient-free situation, the original operator  $\mathcal{D}_C:= \text{div} \big(  \frac{\sigma^2}{2} C \nabla \cdot \big) $ has a trivial invariant measure: uniform distribution, thus no optimization property can be expected when $W$ is not included. We thus turn our results to another way, showing that $\Phi_{\mathcal{D}_C + W}^t p_0$ exponentially converges to a distribution that concentrates on the neighborhood of the global minimum, and as $\sigma \to 0$, the invariant distribution gets more and more concentrate on the global minimum. We now state the result as follows:
\begin{thm}\label{Thm:Gradfree}
    Consider under the space $L^2$, assuming $W$ is bounded on $\mathbb{R}^d$, then $\Phi_{\mathcal{D}_C + W}^t p_0$ converges exponentially to the principle normalized eigenfunction $\mu_\sigma({\bm x})$ of the operator $\mathcal{D}_C + W$. In addition, for any $f \in L^2$ that's smooth compactly supported, $\lim_{\sigma \to 0} \langle\mu_\sigma(\bm x), f\rangle = f(\bm x^\ast) $, where ${\bm x}^\ast$ is the global maximum of $W$, or to say, $\mu_\sigma(\bm x)$ generally 
   converges to $\delta_{\bm x^\ast}(\bm x)$.
\end{thm}



\section{Numerical experiments}
In this section, we first present Fourier Spectral analysis to verify our theoretical analysis of the proposed method. Then, we conduct two inverse problem tests, showing the positive effect as a sampling method when introducing the reweighting/resampling procedure. Finally, we test an optimization task, showing that the resampling technique can help escaping from the local minimum.
\subsection{Fourier spectral method analysis for enhancing spectral gap and concentrating invariant measure }
Let us consider a 1 $d$ periodic problem. We first verify our results in Thm. \ref{Thm:SpectralGapEnhancement}. Let $x \in [0,1)$\footnote{Although our analysis considers the space $\mathbb{R}$, the results can be easily transferred to the periodic case with Asmp. \ref{Asmp:VW} be removed.}, we use Fourier Spectral method \cite{Shen2011SpectralMA} to discretize the differential operator $\mathcal{L}$ and $\mathcal{L} - \varepsilon V$,where
\begin{equation}
    \mathcal{L}f(x) := -\frac{d}{dx}V(x)\frac{d}{dx}f(x) + \frac{d^2}{dx^2} f(x) 
\end{equation}
for any periodic $f \in C^\infty([0, 1))$,  and
\begin{equation}\label{Vfunction}
    V(x) = \cos(9\pi x) - \cos(11\pi x)
\end{equation}
with boundary smoothly modified.
\begin{figure}[h!]\label{SpectralGap}
\subfigure[]{\label{Fig:Eigfunction}
    \includegraphics[width = 0.47\linewidth]{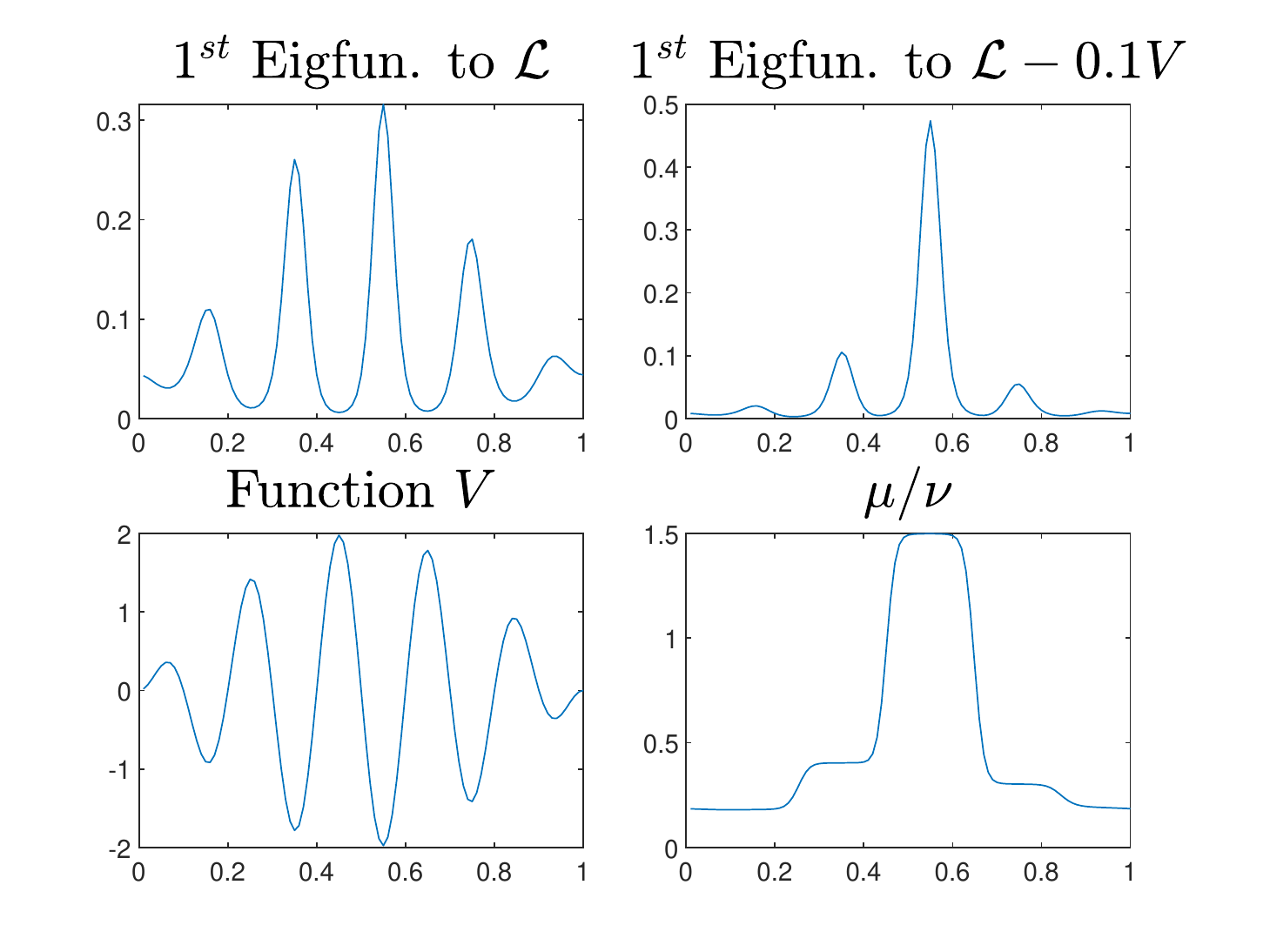}
    }
\subfigure[]{\label{Fig:EigGap}
    \includegraphics[width = 0.47\linewidth]{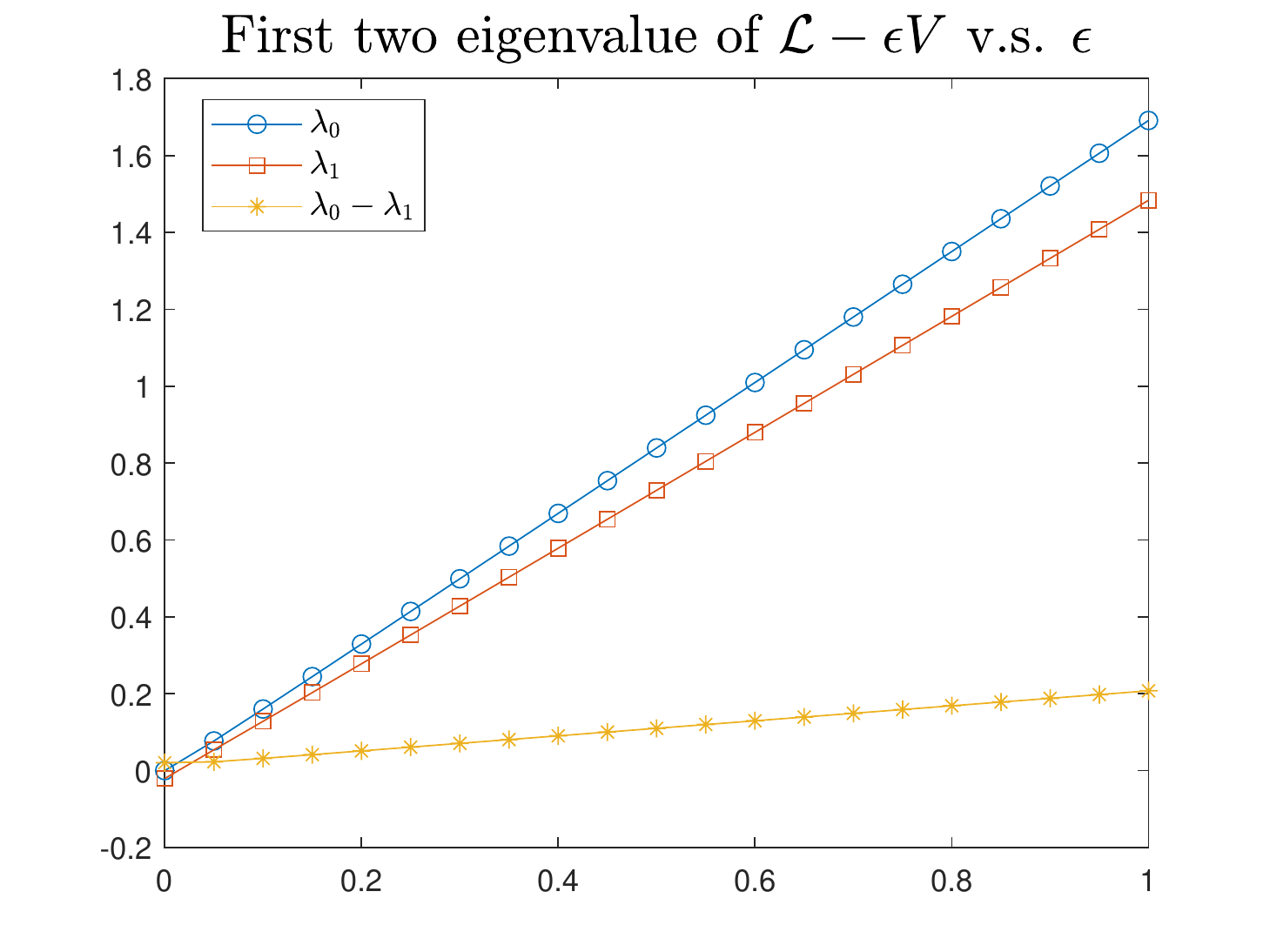}
    }
\vspace{-1em}
\caption{\ref{Fig:Eigfunction} are the graphs to the $1^{st}$ eigfun. $\nu$ to $\mathcal{L}$, $1^{st}$ eigfun. $\mu$ to  $\mathcal{L} - 0.1V$, function $V$, and the quotient of $\mu$ and $\nu$ got by Fourier Spectral method. \ref{Fig:EigGap} is the graph to the first two eigenvalues $\lambda_0(\varepsilon),\lambda_1(\varepsilon)$ of $\mathcal{L} - \varepsilon V$ v.s. $\varepsilon$. }
\vspace{-1em}
\end{figure}

In Fig. \ref{Fig:Eigfunction}, We plot the graphs of the principal eigenfunction $\nu$ to the operator $\mathcal{L}$, the principal eigenfunction $\mu$ to the operator $\mathcal{L}-0.1V$, function $V$ and $\frac{\mu}{\nu}$ respectively. As we expected, $\mu$ concentrates more on the global minimum of $V$ than $\nu$. In Fig. \ref{Fig:EigGap}, we can see $\lambda_0(\varepsilon) - \lambda_1(\varepsilon)$ increasing as $\varepsilon$ increasing, showing the enhancement of spectral gap. This experiment gives a 
visual explanation to Thm. \ref{Thm:SpectralGapEnhancement}: specifically, the eigenfunction  to $\mathcal{L}$ is the Gibbs measure $\nu(\bm x) = \exp(-2V(\bm x)/\sigma^2)$, getting larger when $V $ gets smaller; the eigenfunction $\mu$ to $\mathcal{L}- 0.1V$ has no doubt more mass than $\nu$ that concentrating nearby the global minimum of $V$, and the staircase-like graph of the quotient of $\mu/\nu$ gives a direct explanation to \eqref{eigenfunctionpurturbation}.

Next, let us verify our result in Thm. \ref{Thm:Gradfree}. We use the same space $x\in [0,1)$ and the source term $W(\bm x)  = -V(\bm x)$ where $V(x)$ is the same as in  \eqref{Vfunction}, and test the  analytical property of the following operator
\begin{equation}(\mathcal{D_\sigma} + W)f(x):= \frac{\sigma^2}{2} \frac{d^2}{dx^2}f(x) + W(x) f(x) \end{equation}
for any periodic $f \in C^\infty([0, 1))$. We plot in Fig. \ref{Fig:gradfreeInvariantmeasure} the principal eigenfunction, or invariant distribution density $\mu_\sigma(x)$ of $\mathcal{D}_\sigma$ with different $\sigma$. We also calculate the mass nearby the global minimum: we take the interval $I = [0.44,0.68]$, and calculate $\int_I \mu_\sigma(x) dx$. The results is plotted in Fig. \ref{Fig:gradfreeConcentrateratio}, showing the concentrating tendency when $\sigma \to 0$. This coincides with the result in Thm. \ref{Thm:Gradfree}: the invariant measure concentrates more and more on the global maximum of $W$ as $\sigma \to 0$.

\begin{figure}[h!]
\vspace{-1em}
\subfigure[]{\label{Fig:gradfreeInvariantmeasure}
    \includegraphics[width = 0.47\linewidth]{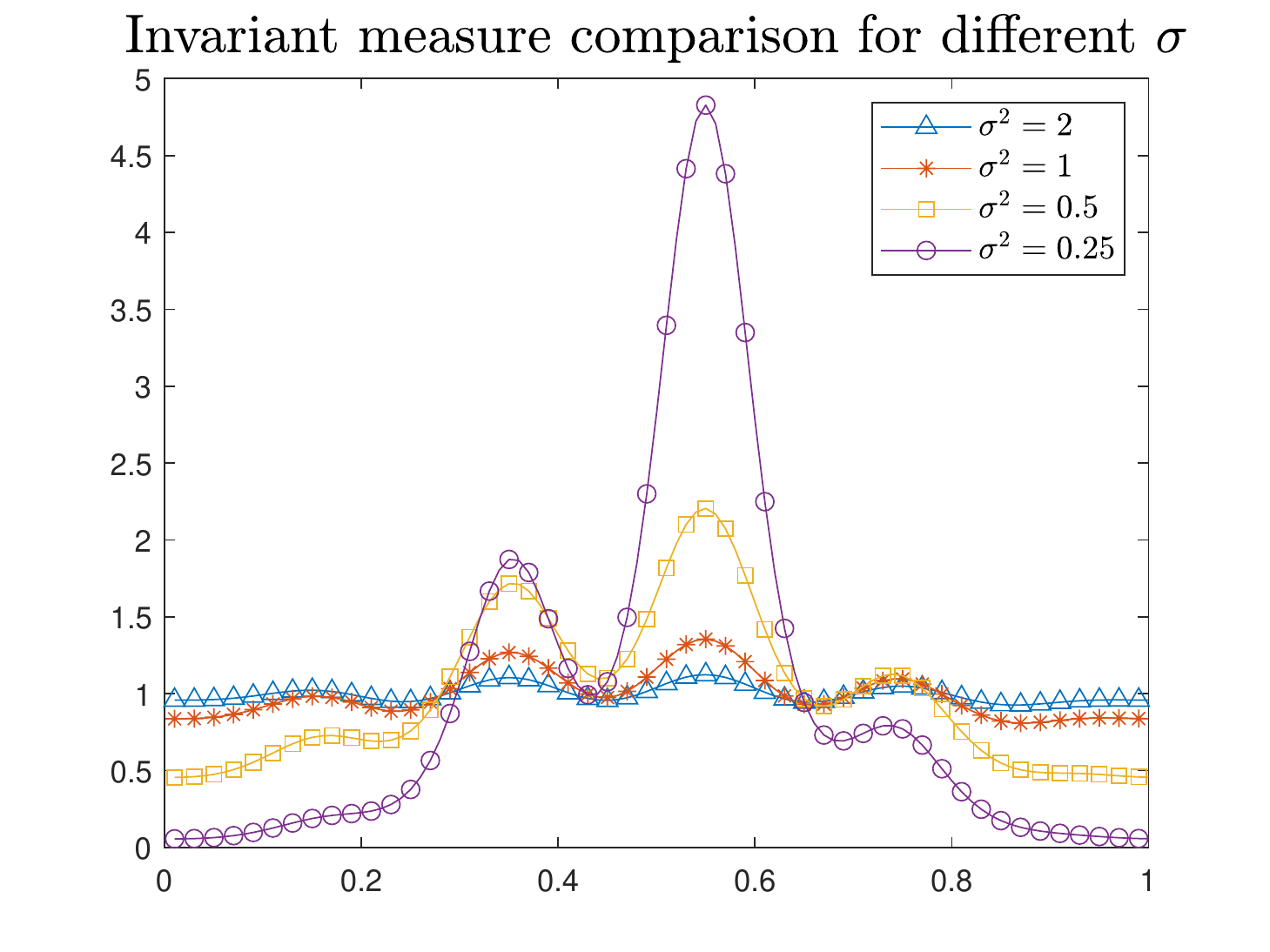}
    }
\subfigure[]{\label{Fig:gradfreeConcentrateratio}
    \includegraphics[width = 0.47\linewidth]{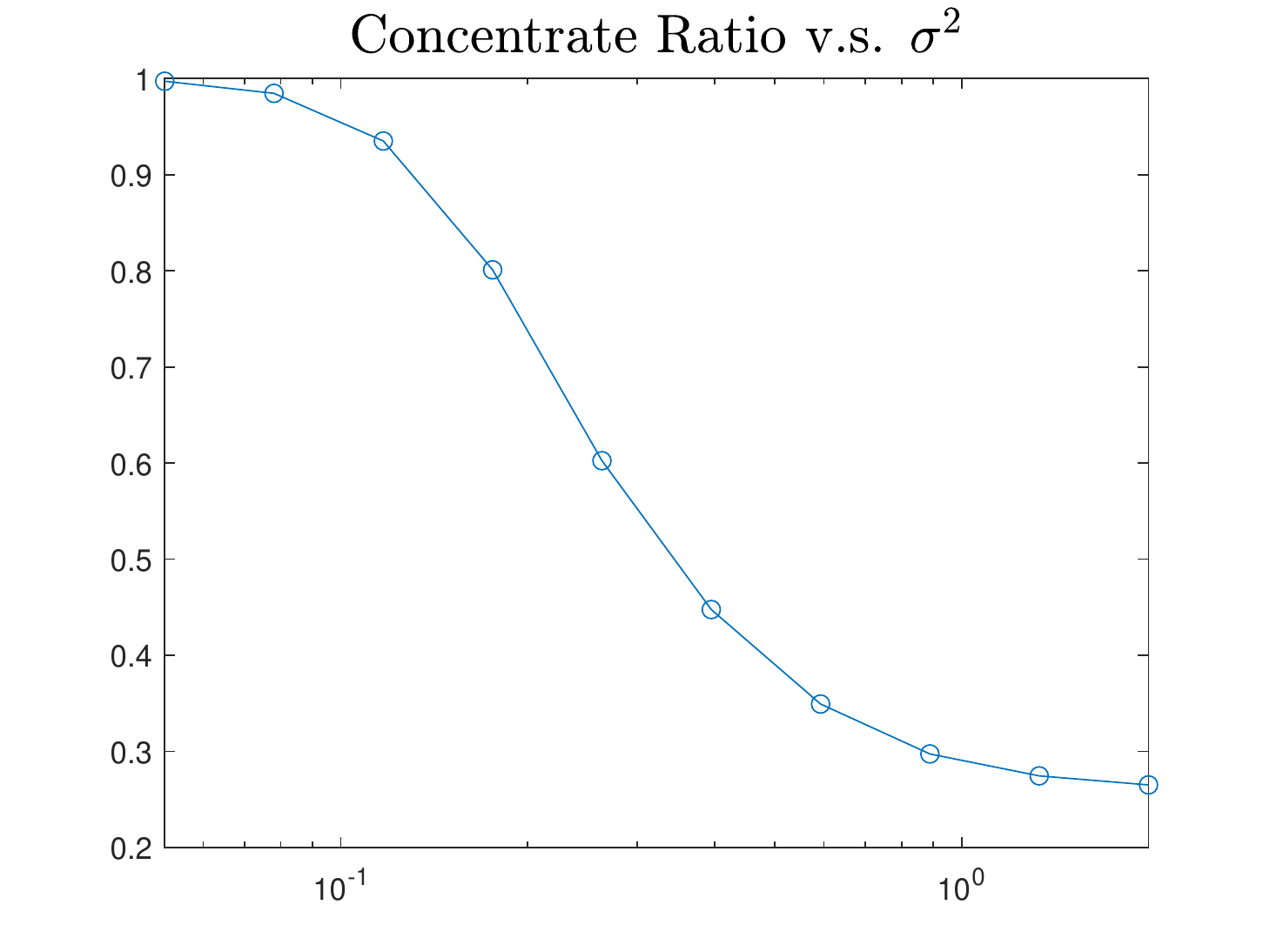}
    }
\vspace{-1em}
\caption{\ref{Fig:gradfreeInvariantmeasure} is the invariant density $\mu_\sigma(x)$  got by Fourier Spectral method under different $\sigma$. \ref{Fig:gradfreeConcentrateratio} plots the ratio $\int_{0.44}^{0.68} \mu_{\sigma}(x) dx$.}
\vspace{-1em}
\end{figure}

\subsection{Numerical tests for inverse problem derivative-free solving and sampling}
In this section, we design some numerical tests in inverse problem solving fields. We compare our RILD algorithm with EKS \cite{GarbunoIigo2019InteractingLD} and EKI \cite{Kovachki2018EnsembleKI}. These problems can be seen as an optimization problem of a least square functional with Tickhonov-Phillips regularization \eqref{Least-square-functional}, thus a derivative-free approximation to the updating rule (\eqref{ls-langevin-diffusion-der-free} for EKS and EKI, or \eqref{rw-langevin-diffusion-der-free} for RILD) can be used to design derivative-free schemes.

We first try to solve a low-dimensional inverse problem. The numerical experiment considered here is the example originally presented by \citet{Ernst2015AnalysisOT}, and also used in \citet{Herty2018KineticMF}. We compare with the result from \citet{GarbunoIigo2019InteractingLD}, and the experimental settings are exactly the same. The forward map is given by the solution of a one-dimensional elliptic boundary value problem as defined in \citet{GarbunoIigo2019InteractingLD}, 
\begin{equation}
    -\frac{d}{du}\left( \exp(x_1) \frac{d}{du} f(u)\right) = 1, u \in [0,1]
\end{equation}
with $f(0) = 0, f(1) = x_2$. The explicit solution is given by
\begin{equation}
    f(u) = x_1 u + \exp(-x_2)\left( -\frac{u^2}{2} + \frac{u}{2}\right).
\end{equation}
Thus we define the forward map 
\begin{equation}\mathcal{G}(\bm x) = \left(f(u_1),f(u_2)\right)^T.\end{equation}
Here $\bm x = (x_1,x_2)^T$ is a constant vector we want to solve, and we have noisy measurements $\bm y = (27.5,79.7)^T$ of $f(\cdot)$ at locations $u_1 = 0.25, u_2 = 0.75$. This can be solved by minimizing the least-square functional defined as in Eq. \ref{Least-square-functional}. We assume observation noise $\Gamma = 0.1^2 I_2$, prior matrix $\Gamma_0 = 10^2 I_2$, and initial ensemble drawn from $N(0,1) \times 
U(90,110).$  The ensemble size is $N = 10^3$. We fix $\sigma = \sqrt{2}$. The stepsize $\tau$ is updated adaptively as in \citet{GarbunoIigo2019InteractingLD}. We take $W(\bm x) = -\Vert \mathcal{G}(\bm x) - \bm y  \Vert^2_\Gamma$. 

We compare our RILD algorithm \ref{alg:RILD} with EKS and EKI algorithms. The key difference between our RILD and EKS in this situation is the use of reweighting/resampling. The results are plotted in Fig. \ref{Fig:EKCompare1}, \ref{Fig:EKCompare2}, and \ref{Fig:EKCompare3}.  From the figure, we can see that our RILD algorithm converges much faster than EKI or EKS algorithms using the same super-parameters, and as our problem is settled as a posterior sampling problem, the ensemble of RILD stops shrinking to the minimum point after some iterations, unless one decrease the diffusion parameter $\sigma$. One could expect, using a decay schedule to $\sigma$, RILD algorithm will perform much better than EKS or EKI algorithms for optimizations.
\begin{figure}[h!]
\vspace{-1em}
\subfigure[]{\label{Fig:EKCompare1}
    \includegraphics[width = 0.3\linewidth]{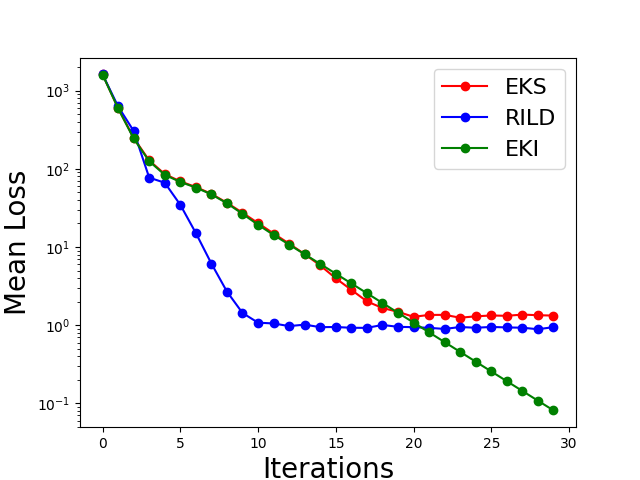}
    }
\subfigure[]{\label{Fig:EKCompare2}
    \includegraphics[width = 0.3\linewidth]{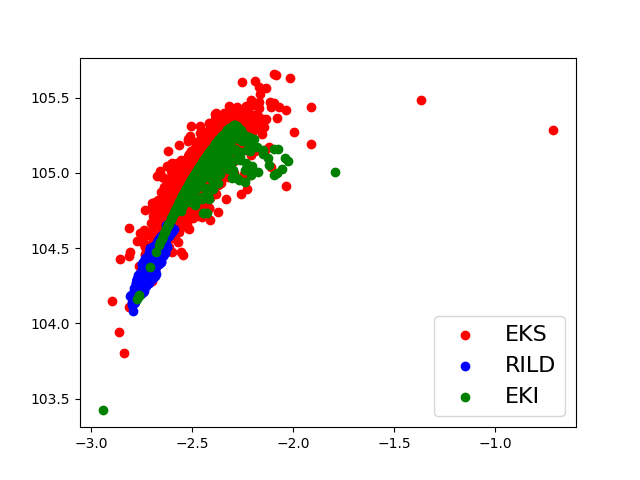}
    }
\subfigure[]{\label{Fig:EKCompare3}
    \includegraphics[width = 0.3\linewidth]{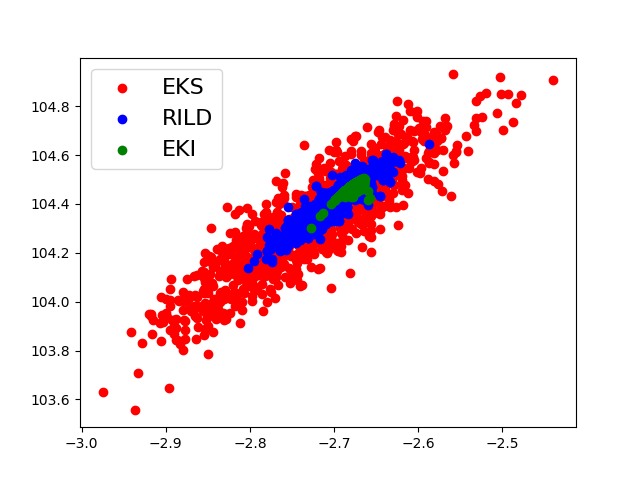}
    }
\vspace{-1em}
\caption{The convergence comparison between RILD, EKS and EKI algorithms. \ref{Fig:EKCompare1} is the mean loss versus iterations, \ref{Fig:EKCompare2} is the ensembles at $15^{th}$ iteration, \ref{Fig:EKCompare3} is the ensembles at $30^{th}$ iteration.}
\vspace{-1em}
\end{figure}

We also tested in a high-dimensional case. Specifically, we define the map $\mathcal{G}: \mathbb{R}^d \to \mathbb{R}^{2d-2}$
\begin{equation*}
    \mathcal{G}(\bm x) = \left(10(x_2 - x_1^2),\cdots,10(x_d - x_{d-1}^2), x_1 , \cdots, x_d \right)^T,
\end{equation*}
$\bm y = (0,\cdots,0, 1, \cdots,1)^T$ with $0$ repeats $d-1$ times and $1$ repeats $d-1$ times. One can verify that
$$\Vert \mathcal{G}(\bm x) - \bm y \Vert^2 = \sum_{i=1}^{d-1} \left(100(x_{i+1} - x_i^2)^2 + (x_i-1)^2\right) $$
is exactly a $Rosenbrock$ function.  We choose $d = 100$, observation noise matrix $\Gamma = 0.1^2 I_{198}$, prior distribution matrix $\Gamma_0 = 10^2 I_{100}$, and initial ensemble drawn from $N(2,0.3^2 I_{100})$. The global solution of $\mathcal{G}(\bm x) = \bm y$ is $\bm x = (1,\cdots,1)^T.$ The ensemble size is fixed to $N = 10^3$. We take $\sigma = \sqrt{2}$ and the stepsize $\tau$ is updated adaptively as before. For RILD algorithm, we choose $W(\bm x) = -5*10^{-3} \Vert \mathcal{G}(\bm x) - \bm y\Vert^2$. Similar to the test before, one can find RILD converges much faster than EKI or EKS in these high-dimensional sampling tasks, and thus perform better in optimization situations. 

\begin{figure}[h!]
\vspace{-1em}
\subfigure[]{\label{Fig:RB1}
    \includegraphics[width = 0.3\linewidth]{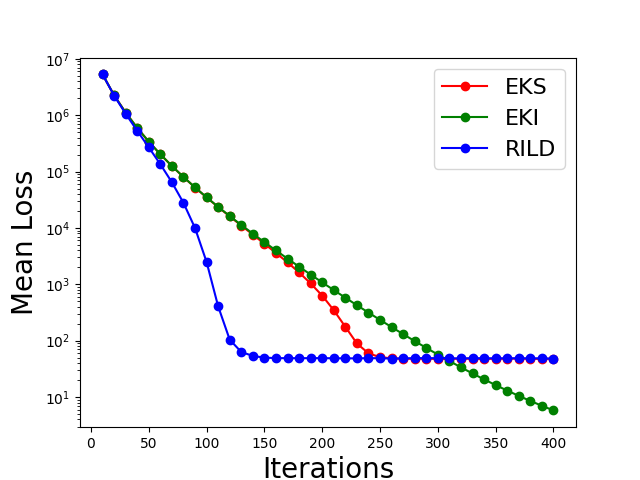}
    }
\subfigure[]{\label{Fig:RB2}
    \includegraphics[width = 0.3\linewidth]{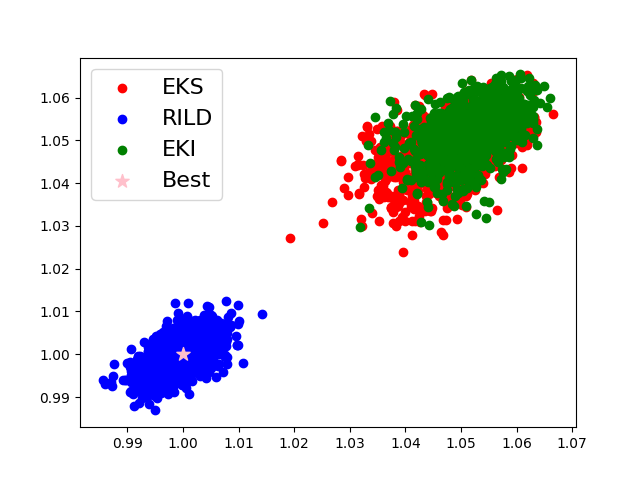}
    }
\subfigure[]{\label{Fig:RB3}
    \includegraphics[width = 0.3\linewidth]{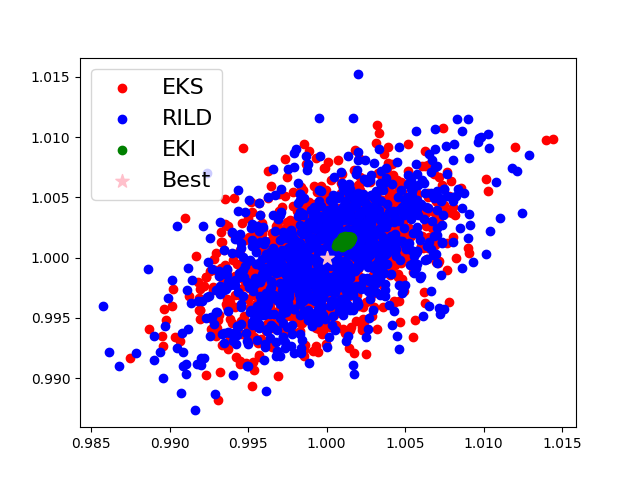}
    }   
\vspace{-1em}
\caption{The comparison between RILD, EKS, and EKI algorithms. \ref{Fig:RB1} is the mean loss v.s. iterrations, \ref{Fig:RB2} is the ensembles at $150^{th}$ iteration, \ref{Fig:RB3} is the ensembles at $400^{th}$ iteration.}
\vspace{-1em}
\end{figure}

\subsection{Numerical tests for highly nonconvex high-dimensional optimization}

We now test our RILD algorithm in a highly nonconvex high-dimensional situation. we test with 100 dimensional $Ackley$ function, which is defined as follows:
\begin{equation}
    V(\bm x) = -a e^ { -b \sqrt{ \frac{1}{d} \sum_{i=1}^{d} x_i^2}} - e^{ \frac{1}{d} \sum_{i=1}^{d} \cos(c x_i)} + a+ e,
\end{equation}
where $\bm x = (x_1,\cdots,x_d)^T, d = 100, a = 20, b = 0.2, c = 2\pi$. As the difficulties mainly raise from the numerious local minimum, covariance modification that is designed for ill-posed problems is not suitable here, we just take $C = I$.

We compare our algorithm RILD in Alg. \ref{alg:RILD} with the classical Gradient Langevin Dynamics (GLD) algorithm. GLD algorithm discretizes a single path of the Langevin Dynamics:
$$\bm x_{n+1} = \bm x_n - \nabla V(\bm x_{n})\tau + \sqrt{\tau \sigma^2} \xi_n. $$
The difference between RILD and GLD is: RILD maintains an ensemble of size $N$  while GLD only maintains 1 individual, and at each step, RILD calculates a weight associated with each individual, then resamples the ensemble according to the weight, see Alg. \ref{alg:RILD}. Now we test if RILD has better ability getting out of local minimums, compared to GLD.

For RILD, we take ensemble size $N = 50$, and randomly pick the initial ensemble\footnote{Such an initial setting creates many difficulties to find the global minimum, as the first term in the $Ackley$ function becomes quickly dominated when $\bm x$ gets away from the origin. } from $N(0,30^2I_{100})$. For GLD, the initial point is randomly chosen from the initial ensemble of RILD. We test a wide range of the stepsize $\tau \in [2,32]$, $\sigma \in [1,16]$, and for each fixed $\tau $ and $\sigma$, we repeat 10 trials to calculate the pass rate: we say one trial is passed, if the RILD or GLD algorithm can find a point $\bm x$ that $V(\bm x)<17$ in $5*10^4$ evaluations\footnote{If one finds a point smaller than 17, the remaining task will be trivial as the first term in the $Ackley$ function gets dominant.}. All trials in all super-parameter settings begin with the same initial ensemble. In \ref{Fig:passrate1} One can see that RILD has a wide range of super-parameter settings to find the true decay directions, while GLD algorithm cannot make it in any tested settings. We also tested GA and PSO under the same initial condition with different super-parameters, and reported the best searching result in Fig. \ref{Fig:passrate2} together with RILD and GLD.

\begin{figure}[h!]
\vspace{-1em}
\subfigure[]{\label{Fig:passrate1}
    \includegraphics[width = 0.6\linewidth, trim = 120 0 180 60 clip]{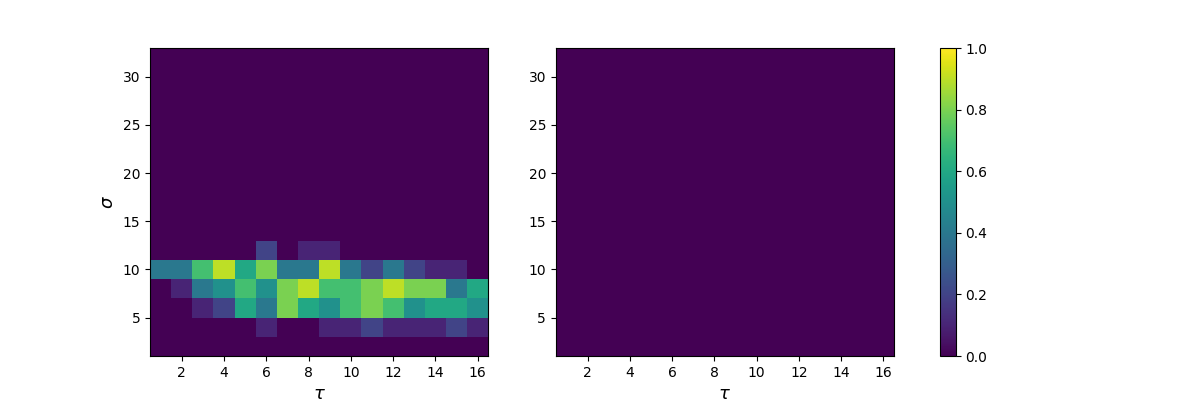}
    }
\subfigure[]{\label{Fig:passrate2}
    \includegraphics[width = 0.34\linewidth]{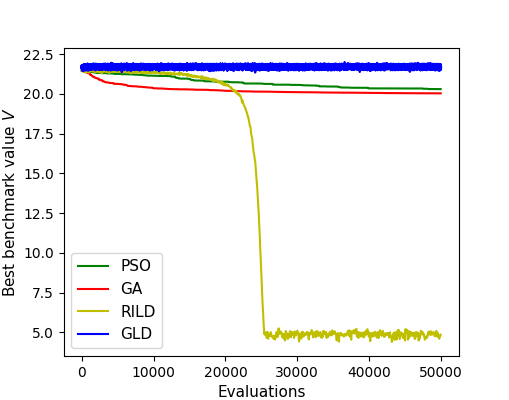}
    }
\vspace{-1em}
\caption{\ref{Fig:passrate1}: pass rates heat map for RILD(left) and GLD(right). \ref{Fig:passrate2}: decay graph for PSO, GA, RILD and GLD where $\tau = 10,\sigma = 5$ for RILD and GLD.  }
\vspace{-1em}
\end{figure}

\section{Conclusion}
In this work, we have demonstrated a methodology for accelerating Langevin Dynamics based algorithms by the addition of the source term $W$ and the use of reweighting/resampling technique -- the RILD algorithm. Our algorithm and analyses shed some light on combining gradient algorithms and genetic algorithms using Partial Differential Equations (PDEs) with provable guarantees.

In the future, we will combine the reweighting technique with higher-order optimization schemes such as momentum accelerated gradient. We will also conduct a finer analysis for convergence with finite particles, which is somehow more important as asymptotic results are only suitable for a large enough ensemble. We expect these studies will bring some insights to design new numerical algorithms.

\clearpage

\bibliography{RILD}
\bibliographystyle{icml2023}
\clearpage
\appendix
\onecolumn
\section{Proofs} \label{sec:proofs}
In this section we give proof to the results in section \ref{theoreticalproperty}.

{\bf Proof  of Lemma ~\ref{Lem:symmetric}}

For any $f,g \in L^2(\nu)\cap C_0^\infty(\mathbb{R}^d)$, note that $W$ is no doubt a self-adjoint operator, we only need to prove for $L_C$. Notice that

\begin{align}
    &Z_\nu\langle f,\mathcal{L}_C  g\rangle_ {L^2(\nu)} = \int_{\mathbb{R}^d} f(\bm x) \left(-\langle C\nabla V(\bm x), \nabla g(\bm x) \rangle + \frac{\sigma^2}{2} \text{div}(C \nabla g(\bm x))\right) e^{-2\sigma^{-2} V(\bm x)} d \bm x \\
    =&\int_{\mathbb{R}^d}  -f(\bm x)\langle C\nabla V(\bm x), \nabla g(\bm x) \rangle e^{-2\sigma^{-2} V(\bm x)} d \bm x + \frac{\sigma^2}{2}\int_{\mathbb{R}^d}  f(\bm x)\text{div}(C \nabla g(\bm x)) e^{-2\sigma^{-2} V(\bm x)}d \bm x \\
    =&\int_{\mathbb{R}^d}  -f(\bm x)\langle C\nabla V(\bm x), \nabla g(\bm x) \rangle e^{-2\sigma^{-2} V(\bm x)} d \bm x - \frac{\sigma^2}{2}\int_{\mathbb{R}^d} \left\langle C\nabla g(\bm x), \nabla( f(\bm x) e^{-2\sigma^{-2} V(\bm x)} )   \right\rangle d \bm x \\
    =&\int_{\mathbb{R}^d}  -f(\bm x)\large(\langle C\nabla V(\bm x), \nabla g(\bm x) \rangle - \langle \nabla V(\bm x), \nabla Cg(\bm x) \rangle\large) e^{-2\sigma^{-2} V(\bm x)} d \bm x - \frac{\sigma^2}{2}\int_{\mathbb{R}^d} \langle C\nabla g(\bm x), \nabla f(\bm x)   \rangle e^{-2\sigma^{-2} V(\bm x)}d \bm x \\
    =& - \frac{\sigma^2}{2}\int_{\mathbb{R}^d} \langle C\nabla g(\bm x), \nabla f(\bm x)   \rangle e^{-2\sigma^{-2} V(\bm x)}d \bm x, 
\end{align}
it is not hard to find that this is a symmetric form, thus $\mathcal{L}_C$ is self-adjoint on $L^2(\nu).$

\QEDA

Next, let us use a changing variable trick to eliminate $C$ for simplicity.

\begin{lem}\label{Lem:affine}
    Given any invertible matrix $A\in \mathbb{R}^{d\times d} $, for any function $f: \mathbb{R}^d \to \mathbb{R}$, we define the function $$\tilde f (\bm x) := f(A\bm x).$$
    We denote the differential operator $\mathcal{\tilde L}_C$ as 
    $$ \mathcal{\tilde L}_C \cdot :=  -\langle C\nabla \tilde V, \nabla \cdot \rangle + \frac{\sigma^2}{2} \text{div}(C\nabla \cdot),$$
    then the following statement is true:
    $$ \mathcal{\tilde L}_{\tilde C} \tilde g (\bm x) = \mathcal{L}_{ C} g (A\bm x),$$
    for any $g: C^\infty(\mathbb{R}^d)$, where $\tilde C = (A^T)^{-1} C A^{-1}$
\end{lem}
\begin{proof}
    Using the chain rule, we notice that
    $$ (\nabla \tilde f)(\bm x) = A (\nabla f)(A\bm x), $$
    thus
    $$\mathcal{\tilde L}_{\tilde C} \tilde g (\bm x) = -\langle A^T\tilde CA(\nabla V)(A\bm x), (\nabla g)(A\bm x) \rangle + \frac{\sigma^2}{2} \text{div}(A^T \tilde C A \nabla g) (A\bm x) = \mathcal{L}_{ C} g (A\bm x).$$
\end{proof}

Lemma \ref{Lem:affine} shows that, we can use simple linear transformation to simplify the operator. Specifically, if we choose $A = C^{\frac{1}{2}}$, we can transfer $\mathcal{L}_C$ to the operator $ \mathcal{\tilde L}$ in the classical form. Thus we may only consider the case when $C = I$. 

To better understand the spectral property of $\mathcal{L}$, we introduce the Witten Laplacian that's unitarily equivalent to $\mathcal{L}$ and $\mathcal{L}^\dagger$, acting on the non-weighted $L^2$-space. The Witten Laplacian is defined by 
$$ \Delta_{f,h} = (-h \nabla + \nabla f) \cdot (h\nabla + \nabla f) = -h^2\Delta + (\vert \nabla f \vert^2 - h\Delta f), $$
and then we can find the following property:
\begin{prop} \label{Prop:Witten}
    The operator $\mathcal{L}$ can be unitarily changed into the Witten Laplacian form as follows:
    $$-2\sigma^{-2} \Delta_{V/2, \sigma^2/2} = \tilde U \mathcal{L} \tilde U^{-1},$$
    where $\tilde U$ is the unitary transformation
    $$ \tilde U : \left\{  \begin{aligned}L^2(\nu) &\to L^2, \\ \phi &\mapsto \phi \sqrt{\nu}. \end{aligned}  \right.  $$
\end{prop}

Thus, we may only consider the spectrum of $$\mathcal{S}:=\tilde U \mathcal{L} \tilde U^{-1}=\frac{\sigma^2}{2} \Delta - \frac{\vert \nabla V \vert^2}{2\sigma^2} + \frac{\Delta V}{2} $$ over $L^2$ in replace of $\mathcal{L}$ over $L^2(\nu)$. Notice that $W$ is invariant under such a transformation, that is, for any $W$, $\tilde U W \tilde U^{-1} = W$, thus the spectrum of $\mathcal{S} + W$ over $L^2$ is equivalent to the spectrum of $\mathcal{L} + W$ over $L^2(\nu)$.

{\bf Proof  of Theorem ~\ref{Thm:convergence1} and \ref{Thm:convergence2}}

By Lemma \ref{Lem:affine} and Proposition \ref{Prop:Witten}, the operator $\mathcal{L} + W$ is transformed into the Schr\"{o}dinger form operator $\mathcal{S}+W$. According to the Assumption \ref{Asmp:VW}, $\mathcal{S}+W$ has discrete real spectrum, and thus  has a sequence of real eigenvalues $\lambda_0 > \lambda_1 \ge \lambda_2\cdots ,$ see \citet{Reed1979MethodsOM}, \citet{Pankov2001IntroductionTS}. 

Recall $$ \Phi^t_{\mathcal{L}+W}(p_0) = \frac{ e^{t(\mathcal{L}^\dagger + W)} p_0}{\int_{\mathbb{R}^d}e^{t(\mathcal{L}^\dagger+W )} p_0(\tilde {\bm x}) d\tilde {\bm x} } =\frac{ e^{t(\mathcal{L}^\dagger + W - \lambda_0)} p_0}{\Vert e^{t(\mathcal{L}^\dagger + W - \lambda_0)}p_0 \Vert_{L^1}} ,$$

Following the same procedure as the Proposition 2. in \citet{Ferr2017ErrorEO}, we have that, for any $p_0 \in L^2(1/\nu)$,
$$ \left\Vert \Phi^t_{\mathcal{L}+W}(p_0) - \frac{\psi_0}{\Vert \psi_0 \Vert_{L^1}} \right\Vert_{L^2(1/\nu)} \le C\left\Vert p_0 - \frac{\psi_0}{\Vert \psi_0 \Vert_{L^1}} \right\Vert_{L^2(1/\nu)} e^{-(\lambda_0 - \lambda_1)t},$$
where $\psi_0$ is the eigenfunction of $\mathcal{L}^\dagger+W$ corresponding to $\lambda_0$. Noticing that, by definition and simple calculation,
$$ \left\Vert p_0 - \frac{\psi_0}{\Vert \psi_0 \Vert_{L^1}} \right\Vert_{L^2(1/\nu)} = \left\Vert \frac{p_0}{\nu} - \frac{\psi_0}{\nu\Vert \psi_0 \Vert_{L^1} } \right\Vert_{L^2(\nu)},$$
and $\frac{\psi_0}{\Vert \psi_0 \Vert_{L^1} \nu}$ is just the normalized eigenfunction $\phi_0$ of $\mathcal{L}+W$ corresponding to $\lambda_0$.

\QEDA

{\bf Proof  of Theorem ~\ref{Thm:SpectralGapEnhancement}}

By Lemma \ref{Lem:affine} and Proposition \ref{Prop:Witten}, we consider the eigenpairs of the operator 
$\mathcal{S}$  and $\mathcal{S}^\varepsilon := \mathcal{S} + \varepsilon m(V)$.
We denote the eigenpair to $\mathcal{S}$  as $\{(\phi_i, \lambda_i), 0 \le i \le \infty\}$, the eigenpair to $\mathcal{S}^\varepsilon$ as $\{(\phi^\varepsilon_i, \lambda^\varepsilon_i), 0 \le i \le \infty\}$. For simplicity we denote $\beta = \frac{2}{\sigma^2}$. The perturbation theory of spectrum (see Chapter 5.4, \cite{Pankov2001IntroductionTS} or \citet{Kato1966PerturbationTF}) reads:
\begin{equation} \lambda^\varepsilon_i = \lambda_i + \varepsilon \frac{\langle\phi_i, m(V) \phi_i\rangle}{\langle\phi_i,\phi_i\rangle} + O(\varepsilon^2), \label{eigenvaluepurturbation}\end{equation}
\begin{equation}\phi_i^\varepsilon = \phi_i + \varepsilon \sum_{j \neq i} \frac{\langle \phi_j, m(V) \phi_i\rangle }{(\lambda_i - \lambda_j)\langle \phi_j,\phi_j\rangle} \phi_j .\label{eigenfunctionpurturbation}\end{equation}
We first prove that, for small enough $\varepsilon$,

$$ \lambda_0^\varepsilon - \lambda_1^\varepsilon > \lambda_0 - \lambda_1,$$ by \eqref{eigenvaluepurturbation} this equivalent to
$$\frac{\langle \phi_0, m(V) \phi_0\rangle }{\langle \phi_0,\phi_0\rangle} > \frac{\langle \phi_1, m(V) \phi_1\rangle}{\langle \phi_1,\phi_1\rangle }.$$

According to Chapter 2.5 in \citet{Lelivre2016PartialDE} and \citet{Nier2004QuantitativeAO}, there are exactly $m_0$ eigenvalues close enough to $\lambda_0$ corresponding to the $m_0$ local minimums of $V$, and for $1 \le j \le m_0$, $\phi_i$ are of the form $\chi_j \exp (-\beta V/2)$, where the functions $\chi_j$ are locally constant over the basins of attractions of the local minima of $V$, we denote the minimum corresponding to each $\phi_j$ as $\bm x_j$.

We then denote the attraction basin as $B_1$ with corresponding local minimum $V(\bm x_1)$; and denote $M = \max_{\bm x \in B_1} V(\bm x)$. According to \citet{Nier2004QuantitativeAO}, for any $\delta >0$, the region $G_1 := \{ V(\bm x) \le M - \delta \}\cap B_1^c$, $\chi_1$ is asymptotically constant as $\beta \to 0$. Thus we may assume
    \begin{equation} \chi_1|_{B_1\cup G_1} \approx I_{B_1} + c(\beta) I_{G_1} . \label{chi}\end{equation}
    By observing that when $\bm x \in (B_1 \cup G_1)^c, V(\bm x) > M - \delta$, thus we have
    $$0 = \langle \phi_1, \phi_0 \rangle = \int_{\mathbb{R}^n} \chi_1 \exp(-\beta V)d\bm x = \int_{B_1} \exp(-\beta V) d\bm x +  c(\beta) \int_{G_1} \exp(-\beta V) d\bm x + O(e^{-\beta(M - \delta)})$$
    and $\bm x_0 \in G_1$, we decude that 
    \begin{equation} c(\beta) = - \frac{\int_{B_1} \exp(-\beta V) d\bm x + O(e^{-\beta(M - \delta)})}{\int_{G_1} \exp(-\beta V) d\bm x} = -O(e^{-\beta(V(\bm x_1)- V(\bm x_0)) } ).\label{cbeta} \end{equation}
    Then 
    \begin{align*}\frac{\langle \phi_1, m(V) \phi_1\rangle}{\langle \phi_1,\phi_1\rangle } &= \frac{\int_{B_1} m(V) \exp(-\beta V) d\bm x   -O(e^{-2\beta(V(x_1)- V(x_0)) } ) \int_{G_1} m(V) \exp(-\beta V) d\bm x}{\int_{B_1} \exp(-\beta V) d\bm x  -O(e^{-2\beta(V(x_1)- V(x_0)) } ) \int_{G_1} \exp(-\beta V) d\bm x} \\
    &= m(V(\bm x_1)) + O(\beta^{-1}),\end{align*}
    as $m(V)$ has no dependence on $\beta$.

    On the other side, 
    $$\frac{\langle \phi_0, m(V) \phi_0\rangle}{\langle \phi_0,\phi_0\rangle } = m(V(\bm x^\ast)) + O(\beta^{-1})$$
    thus we arrive at $$\frac{\langle \phi_0, m(V) \phi_0\rangle }{\langle \phi_0,\phi_0\rangle} > \frac{\langle \phi_1, m(V) \phi_1\rangle}{\langle \phi_1,\phi_1\rangle },$$
    because $m$ is a decreasing function and $V(\bm x^\ast) < V(\bm x_1).$

    By \eqref{eigenfunctionpurturbation}, for the eigenfunction estimation of $\phi_0^\varepsilon$, we aim at proving that
    \begin{equation}\langle \phi_j, m(V) \phi_0\rangle <0. \label{eigfuncondition}\end{equation}
    for $j$ that $\lambda_j$ close enough to $\lambda_0$, thus these term will suppress the height nearby local minimum of $\phi_0$, remaining the height nearby global minimum.
    
    We only prove this for $j = 1$, and for $i = 2,\cdots,m_0$, the procedure are similar. For $j > m_0$, these term in \eqref{eigenfunctionpurturbation} are asynmtotically negotiable.
    
    Note that, by omitting small term (the integration region outside $B_1\cup G_1$) and use \eqref{chi}, \eqref{cbeta} , the statement \eqref{eigfuncondition} is equivalent to 
    $$ \int_{B_1} m(V) e^{-\beta V} d\bm x  - \frac{\int_{B_1} \exp(-\beta V) d\bm x }{\int_{G_1} \exp(-\beta V) d\bm x}\int_{G_1} m(V) e^{-\beta V} d\bm x < 0  $$
    
    $$\frac{\int_{B_1}m(V) e^{-\beta V}d\bm x}{\int_{B_1}e^{-\beta V}d\bm x}<  \frac{\int_{G_1} m(V)e^{-\beta V}d\bm x}{\int_{G_1} e^{-\beta V}d\bm x} $$
    which is asymptotically true when $\beta \to \infty$, as left hand side trends to $m(V(\bm x_1))$ while right hand side trends to $m(V(\bm x^\ast))$.
    
    \QEDA

{\bf Proof  of Theorem ~\ref{Thm:Gradfree}}

We begin with considering the spectrum of $W$. Let $W_{min} = \inf\{W(\bm x)\}, W_{max} = \sup \{W(\bm x)\}$ and $W_{min}$ can be $-\infty$. We assume $W_{max}$ is reachable with one and only one maximum $\bm x^\ast$. 

Note that for any $\lambda > W_{max}$ or $\lambda < W_{min}$, the operator $\lambda - W$ is invertible in $L^2$, thus the spectrum of $W$ is just $[W_{min},W_{max}]$. We now consider the operator $ \frac{\sigma^2}{2} \Delta + W$. By Rayleigh-quotient formula \cite{Parlett1981TheSE}, the principle eigenvalue $\lambda_0(\sigma)$ of $\frac{\sigma^2}{2}\Delta + W$ over $L^2$ is

$$\lambda_0(\sigma) = \sup_{\phi \in H^1} \frac{\int_{\mathbb{R}^d} \phi ((\frac{\sigma^2}{2}\Delta + W) \phi) d \bm x}{\int_{\mathbb{R}^d} \phi^2 d \bm x} = \sup_{\phi \in H^1} (\frac{\int_{\mathbb{R}^d} \phi^2 W d \bm x}{\int_{\mathbb{R}^d} \phi^2 d \bm x} - \frac{\sigma^2}{2}\frac{\int_{\mathbb{R}^d} (\nabla \phi)^T \nabla \phi  d\bm x}{\int_{\mathbb{R}^d} \phi^2 d \bm x}) $$

It's worthwhile noticing that $\lambda_0(\sigma)$ is a decreasing function respect to $\sigma$ , thus for any $\sigma > 0$, $\lambda_0(\sigma) < \lambda_0(0) = W_{max}$.

Take a series of test function $\phi_n (\bm x) \in H^1$ that gradually trends to the delta function $\delta_{\bm x^\ast}(\bm x)$, by noticing that if we take $\sigma_n :=  \sqrt{ 2 \frac{1}{n} \frac{\langle \phi_n, \phi_n\rangle}{\langle \nabla\phi_n, \nabla\phi_n\rangle}}$, we have

$$\lim_{\sigma \to 0} \lambda_0(\sigma) \ge \lim_{n\to\infty}(\frac{\langle \phi_n, W\phi_n\rangle}{\langle \phi_n, \phi_n\rangle} - \frac{\sigma^2_n}{2}\frac{\langle \nabla\phi_n, \nabla\phi_n\rangle}{\langle \phi_n, \phi_n\rangle}) = W_{max}  $$ 

Thus $\lambda_0(\sigma)$ is continuous at $\sigma = 0$.

Now denote the principal normalized eigenfunction of $\frac{\sigma^2}{2} \Delta + W$ is $\psi_0(\sigma)$. For any $f \in L^2$ that's  smooth
compactly supported, noticing that

\begin{align}\langle \psi_0(\sigma) , f \rangle &= \lambda_0(\sigma)^{-1} \langle  (\frac{\sigma^2}{2} \Delta + W)\psi_0(\sigma) , f \rangle\\
& =\lambda_0(\sigma)^{-1} \frac{\sigma^2}{2} \langle \psi_0(\sigma), \Delta f \rangle + \lambda_0(\sigma)^{-1} \langle  \psi_0(\sigma), W f \rangle,
\end{align}

taking $\sigma \to 0$, the first term trends to 0, and thus

$$\lim_{\sigma \to 0}\langle \psi_0(\sigma) , f \rangle  =   \lim_{\sigma \to 0}\langle \frac{W}{W_{max}}\psi_0(\sigma) , f \rangle$$

This implies that as  $\sigma \to 0,  \psi_0(\sigma)$ acts closer and closer to $\delta_{\bm x^\ast}(\bm x)$.

\QEDA
\end{document}